\documentclass[twoside]{article}
\usepackage[accepted]{aistats2025}
\usepackage{booktabs}
\usepackage{algorithm}
\usepackage{algpseudocode}
\usepackage{multirow}

\usepackage[round]{natbib}

\usepackage{amsmath, amsfonts,amssymb}
\usepackage{graphicx}
\usepackage[colorlinks=true, allcolors=blue]{hyperref}
\usepackage{xcolor}
\usepackage{url}
\usepackage{bm}
\usepackage{float} 
\usepackage{mathtools} 

\usepackage{amsthm}
\usepackage[english]{babel}
\usepackage{pgf}
\usepackage{amssymb,amsmath,mathtools}
\usepackage{lipsum}
\usepackage{enumitem}
\setlist[itemize,enumerate]{topsep=4pt,itemsep=0pt,leftmargin=*}

\DeclareGraphicsExtensions{.pdf,.png}

\DeclareMathOperator*{\softmax}{\mathsf{softmax}}
\DeclareMathOperator*{\sparsemax}{\mathsf{sparsemax}}
\DeclareMathOperator*{\entmax}{\mathsf{entmax}}

\usepackage{mdframed}
\definecolor{theoremcolor}{rgb}{0.94, 0.97, 1.0}
\definecolor{examplecolor}{rgb}{0.94, 0.97, 1.0}
\mdfsetup{
	innertopmargin=8pt,
	innerbottommargin=8pt,
	leftmargin=4pt,
	rightmargin=4pt,
	backgroundcolor=theoremcolor,
	linewidth=0pt,
}
\usepackage{xfrac}

\newcommand{\entalpha}{\gamma}
\newcommand{\invprob}{\mathsf{InvProb}}
\newcommand{\raps}{\mathsf{RAPS}}

\newcommand{\logmargin}{\mathsf{log\text{-}margin}}
\newcommand{\optentmax}{\mathsf{opt\text{-}entmax}}

\newmdtheoremenv[linewidth=0pt,innerleftmargin=4pt,innerrightmargin=4pt]{definition}{Definition}
\newmdtheoremenv[linewidth=0pt,innerleftmargin=4pt,innerrightmargin=4pt]{proposition}{Proposition}
\newmdtheoremenv[linewidth=0pt,innerleftmargin=0pt,innerrightmargin=0pt,backgroundcolor=examplecolor]{example}{Example}
\newmdtheoremenv{corollary}{Corollary}
\newmdtheoremenv{theorem}{Theorem}
\newmdtheoremenv{lemma}{Lemma}
\usepackage{comment}
\newcommand{\margarida}[1]
{{\color{teal}\textbf{[MC: #1]}}}

\newcommand{\andre}[1]
{{\color{red}\textbf{[AM: #1]}}}
\newcommand{\rebuttal}[1]
{{\color{black}{#1}}}
\newcommand{\sophia}[1]
{{\color{green!80!black}\textbf{[SS: #1]}}}
\begin{document}

%

%
\runningauthor{Margarida M. Campos, João Cálem, Sophia Sklaviadis, Mário A.T. Figueiredo, André F.T. Martins}
\twocolumn[

\aistatstitle{Sparse Activations as Conformal Predictors}

\aistatsauthor{ Margarida M. Campos$^{1,2}$ \And João Cálem$^{2}$ \And  Sophia Sklaviadis$^{1,2}$ \AND Mário A.T. Figueiredo$^{1,2,3}$ \And André F.T. Martins$^{1,2,3,4}$ }

\aistatsaddress{ $^1$ Instituto de Telecomunicações \\ $^2$ Instituto Superior Técnico, Universidade de Lisboa  \\ $^3$ ELLIS Unit Lisbon \\ $^4$ Unbabel \\ \texttt{margarida.campos@tecnico.ulisboa.pt}} ]

\begin{abstract}
Conformal prediction is a distribution-free framework for uncertainty quantification that replaces point predictions with sets, offering marginal coverage guarantees (\textit{i.e.}, ensuring that the prediction sets contain the true label with a specified probability, in expectation). 
In this paper, we uncover a novel connection between conformal prediction and sparse ``softmax-like'' transformations, such as sparsemax and $\entalpha$-entmax (with $\entalpha > 1$), which may assign nonzero probability only to a subset of labels. 
We introduce new non-conformity scores for classification that make the calibration process correspond to the widely used temperature scaling method. At test time, applying these sparse transformations with the calibrated temperature leads to a support set (\textit{i.e.}, the set of labels with nonzero probability) that automatically inherits the coverage guarantees of conformal prediction. 
Through experiments on computer vision and text classification benchmarks, we demonstrate that the proposed method achieves competitive results in terms of coverage, efficiency, and adaptiveness compared to standard non-conformity scores based on softmax.%
\end{abstract}

\begin{figure}
\centering
\includegraphics[width=1\columnwidth,trim = {2cm 0cm 3.5cm 0cm}]{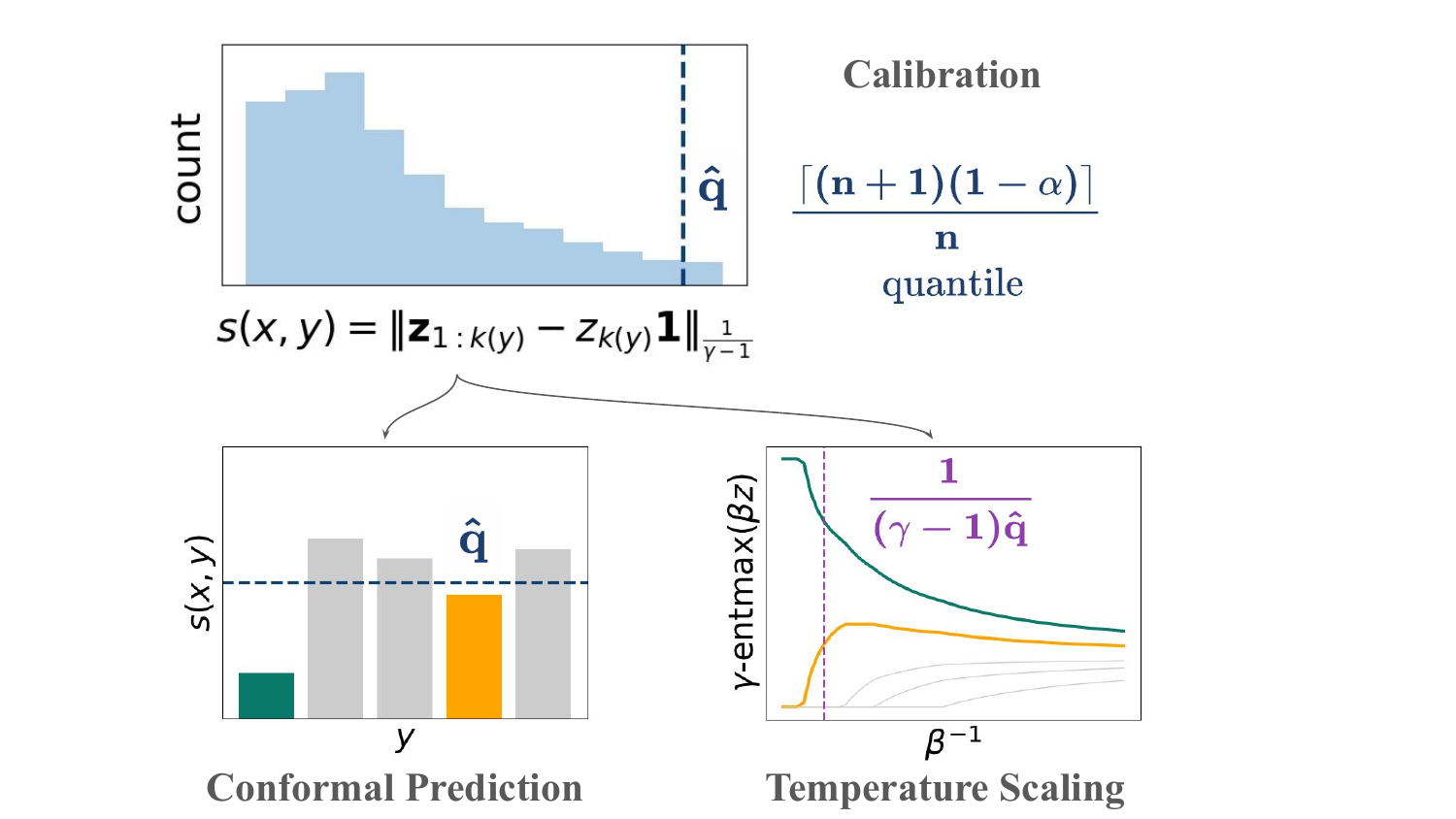}
\caption{\textbf{Conformal prediction meets temperature scaling:} we derive new non-conformity scores $s(x, y)$ that make conformal prediction equivalent to $\entalpha$-entmax temperature scaling.}
\label{fig:diagram}
\end{figure}
\section{INTRODUCTION}

The use of learned predictive models in many high-stakes applications (\textit{e.g.}, medical, legal, or financial) has stimulated extensive research on uncertainty quantification as a way to enhance predictions with reliable confidence estimates \citep{Silva_Filho,Gawlikowski}. In classification tasks, this corresponds to providing an estimate of the posterior probabilities of each of the classes, given the sample to be classified. In regression, this would correspond to providing, not only a point estimate, but also a confidence interval around that estimate.

Unfortunately, modern deep networks are not well calibrated: the class probability estimates may be significantly different from the posterior probability, and thus cannot be trusted as confidence levels \citep{guo2017calibrationmodernneuralnetworks,wenger2020}. This observation has motivated considerable work on developing methods to obtain well-calibrated class probability estimates.


To address this problem, 
\textbf{conformal prediction} (\S\ref{subsec:background_cp}) has emerged as a powerful framework for uncertainty quantification that comes with strong theoretical guarantees under minimal assumptions, as it is model-agnostic and distribution-free \citep{vovk_book,angelopoulos2022gentleintroductionconformalprediction}. Conformal predictors are \textit{set} predictors, \textit{i.e.}, their outputs are either sets of labels (for classification tasks) or intervals (for regression tasks). Most common variants, namely \textit{split} conformal prediction  \citep{papadopoulos2002inductive}, do not require model retraining: prediction rules are created using a separate calibration set through the design of a non-conformity score, which is a measure of how unlikely
an input-output pair is.  

A parallel research direction concerns \textbf{sparse} alternatives to the softmax activation  (\S\ref{subsec:entmax}), \textit{i.e.}, capable of producing probability outputs where some labels receive zero probability, yet are differentiable, thus usable in backpropagation \citep{martins2016softmaxsparsemaxsparsemodel,peters2019sparsesequencetosequencemodels}. The class of Fenchel-Young losses \citep{blondel2020learningfenchelyounglosses} provides a framework to learn such sparse predictors, generalizing the softmax activation and the cross-entropy loss. 
An added benefit of sparse probability outputs is that the labels with nonzero probability can be directly interpreted as a \textit{set prediction}, eliminating the need for creating these sets by thresholding probability values. Such set predictions have been used by \citet{martins2016softmaxsparsemaxsparsemodel} for multi-label classification, but not in the scope of uncertainty quantification. 

This paper connects the two lines of work above by \textbf{conformalizing sparse activations} (Figure~\ref{fig:diagram}; \S\ref{sec:proposed}). 
We focus on the $\entalpha$-entmax family%
\footnote{Called $\alpha$-entmax by \citet{peters2019sparsesequencetosequencemodels}. We choose $\entalpha$ in this paper to avoid clash with the confidence level $\alpha$, commonly used in the conformal prediction literature.} %
\citep{peters2019sparsesequencetosequencemodels}, which includes softmax ($\entalpha=1$) and sparsemax ($\entalpha=2$) as particular cases, and is sparse for $\entalpha > 1$. We create set predictions by including all labels receiving nonzero probability under those activations---these sets can be made larger or smaller by controlling a \textit{temperature} parameter. We leverage the calibration techniques of conformal prediction to choose the temperature that leads to the desired coverage level. 

Our main contributions are the following: 
\begin{itemize}
    \item A new non-conformity score that makes the conformal predictor equivalent to temperature scaling of the sparsemax activation, establishing a formal link between the two lines of work and ensuring statistical coverage guarantees (\S\ref{subsec:sparsemax_cp}). 
    \item Generalization of the construction above to the $\entalpha$-entmax family of activations, with the same guarantees (\S\ref{subsec:entmax_cp}). This yields a calibration strategy that can also be used with softmax ($\entalpha=1$) and where $\entalpha$ can be tuned as a hyperparameter (\textit{e.g.} to optimize prediction set size). 
    \item Empirically validation of our method on a range of computer vision and text classification tasks across several datasets, comparing the resulting coverage,  efficiency, and adaptiveness with commonly used non-conformity scores (\S\ref{sect:experiments}).
    \rebuttal{
    \item Expansion of the family of non-conformity scores for conformal prediction tasks with new scores that are competitive with state-of-the-art techniques.
    }
\end{itemize}

\rebuttal{
The code to reproduce all the reported experiments is publicly available.%
\footnote{\scriptsize{\url{https://github.com/deep-spin/sparse-activations-cp}}}
}

\section{BACKGROUND}
\label{sec:background}
This section provides background by briefly reviewing basic concepts of conformal prediction and sparse activations.

\subsection{Conformal Prediction}
\label{subsec:background_cp}


Consider a supervised learning task with input set $\mathcal{X}$, output set $\mathcal{Y}$, and a predictor, $f:\mathcal{X}\rightarrow \mathcal{Y}$. 
Given a new observation, $x_\text{test} \in \mathcal{X}$, a standard predictor produces a single point estimate: $\hat{y}_\text{test} = f(x_\text{test})$. A conformal predictor, $\mathcal{C}_\alpha: \mathcal{X}\rightarrow 2^{\mathcal{Y}}$, outputs instead a prediction set $\mathcal{C}_\alpha(x_\text{test}) \subseteq \mathcal{Y}$, which is expected to include the correct target, $y_\text{test}$, with a given probability $1-\alpha$, where $\alpha$ is a user-chosen error rate. 

\paragraph{Procedure} We focus on split conformal prediction, which does not involve model retraining \citep{papadopoulos2002inductive}. 
Given the predictor, $f$, and a collection of calibration samples, $\mathcal{D}_\text{cal} = \bigl((x_1,y_1),..., (x_n,y_n)\bigr)$, a conformal predictor is defined by a non-conformity score $s: \mathcal{X}\times\mathcal{Y}\rightarrow \mathbb{R}$. The non-conformity score measures how unlikely an input-output pair $(x,y) \in \mathcal{X}\times\mathcal{Y}$ is, compared to the remaining calibration data. Ideally, predictions $\hat{y}\in \mathcal{Y}$ yielding pairs $(x_\text{test}, \hat{y})$ that are likely to occur in the data should have a low non-conformity score, and should thus be included in the prediction set $\mathcal{C}_\alpha(x_\mathrm{test})$.

The procedure for generating the prediction set for an unobserved sample, $x_\text{test}$, is as follows: 
\begin{enumerate}
    \item During calibration, the non-conformity scores are computed for $\mathcal{D}_\mathrm{cal}$, $(s_1, ... , s_n)$, where $s_i = s(x_i,y_i)$. \item  A threshold $\hat{q}$ is set to the $\lceil(n+1)(1-\alpha)\rceil/n$ empirical quantile of these scores.
    \item At test time, the following prediction set is output:
    \begin{align}\label{eq:prediction_set}
        \mathcal{C}_\alpha(x_\text{test})=\{y\in\mathcal{Y}: s(x_\text{test},y)\leq \hat{q}\}.
    \end{align}
\end{enumerate}

\paragraph{Coverage guarantees} 
If the calibration data and the test datum $((X_1, Y_1), ..., (X_n, Y_n), ((X_{\text{test}}, Y_{\text{test}}))$ are \textit{exchangeable},\footnote{A sequence $(Z_1,...,Z_n)$ is said to be \textit{exchangeable} if $
		\bigl(Z_1,...,Z_n\bigr) \overset{d}{=} \bigl( Z_{\pi(1)},...,Z_{\pi(n)}\bigr)$
	for any permutations $\pi$ of $\{1,...,n\}$, where $\overset{d}{=}$ stands for \textit{identically distributed}.} conformal predictors obtained with this procedure yield prediction sets satisfying
\begin{equation}
	\mathbb{P}\big( Y_\text{test} \in \mathcal{C}_\alpha(X_\text{test})\big)\geq 1- \alpha,
	\label{eq:coverage}
\end{equation} 
as proved by \citet{vovk_book}, without any other assumptions on the predictive model or the data distribution. 
Ideally, we would like the predictor with the required coverage to be \textit{efficient}, \textit{i.e.}, to have a small average prediction set size $\mathbb{E} [|C_\alpha(X)|]$, and \textit{adaptive}, \textit{i.e.}, instances that are harder to predict should yield larger prediction sets, representing higher uncertainty. 

\subsection{Sparse Activations}
\label{subsec:entmax}

Consider a classification task with $K$ classes, $|\mathcal{Y}| = K$, and let $\triangle_K = \{ \bm{p} \in \mathbb{R}^K : \bm{p} \geq \mathbf{0}, \mathbf{1}^\top \bm{p} = 1 \}$ denote the $(K-1)$–dimensional probability simplex. Typically, predictive models output a vector of label scores, $\bm{z}\in \mathbb{R}^K$, which is converted to a probability vector in $\triangle_K$ through a suitable transformation, usually softmax:  
\begin{equation}\label{eq:softmax}
\softmax(\bm{z})_j := \frac{\exp(z_j)}{\sum_i \exp(z_i)}.
\end{equation}
Since the exponential function is stricly positive, the softmax transformation never outputs any zero. 
\citet{martins2016softmaxsparsemaxsparsemodel} introduced an alternative transformation,  \textit{sparsemax}, which can produce sparse outputs while being almost-everywhere differentiable:
\begin{equation}\label{eq:sparsemax}
\sparsemax(\bm{z}) := \arg\min_{\bm{p} \in \triangle_K} \| \bm{p} - \bm{z} \|^2.
\end{equation}
The solution of \eqref{eq:sparsemax}, which corresponds to the Euclidean projection of $\bm{z}$ onto $\triangle_K$, can be computed efficiently through Algorithm~\ref{alg:sparsemax}. 
This algorithm, which sorts the label scores and seeks a threshold which ensures normalization,  motivates our proposed non-conformity score to be presented in \S\ref{subsec:sparsemax_cp}. 

\begin{algorithm}[t]
\caption{Sparsemax evaluation.}
\label{alg:sparsemax}
\begin{algorithmic}[1]
\Require $\bm{z} \in \mathbb{R}^K$
\State Sort $z_{(1)} \ge z_{(2)} \ge ... \ge z_{(K)}$
\State Find $k(\bm{z}) = \max \left\{j \in [K] : 1 + jz_{(j)} > \sum_{k=1}^j z_{(k)} \right\}$. \label{alg:sparsemax_line_k}
\State Compute $\tau = \frac{\left(\sum_{k=1}^{k(\bm{z})} z_{(k)} \right) - 1}{k(\bm{z})}$. 
\State \Return $\mathsf{sparsemax}(\bm{z}) = (\bm{z} - \tau \bm{1})_+$. 
\end{algorithmic}
\end{algorithm}

As shown by \citet{peters2019sparsesequencetosequencemodels} and \citet{blondel2020learningfenchelyounglosses}, sparsemax and softmax are particular cases of a family of $\entalpha$-entmax transformations, defined as
\begin{align}\label{eq:entmax}
	\entalpha\text{-}\entmax(\bm{z}) := \arg\max_{\bm{p} \in \triangle_K} \bm{p}^\top \bm{z} + H_\entalpha(\bm{p}), 
\end{align}
where $H_\entalpha$ is a generalized entropy function \citep{tsallis1988}, defined, for $\entalpha > 0$, as
\begin{equation}
H_\entalpha(\bm{p}) = 
\begin{cases} 
	\frac{1}{\entalpha(\entalpha - 1)} \left( 1 - \sum_{j=1}^K p_j^\entalpha \right), & \entalpha \neq 1, \\
	-\sum_{j=1}^K p_j \log p_j, & \entalpha = 1. 
\end{cases}
\end{equation} 
This family of entropies is continuous with respect to $\entalpha$ and recovers the Shannon entropy for $\entalpha = 1$. 
Softmax and sparsemax are particular cases of the $\entalpha$-entmax transformation \eqref{eq:entmax} for $\entalpha=1$ and $\entalpha=2$, respectively. 
Crucially, setting $\gamma>1$ enables sparsity in the output, \textit{i.e.}, for certain inputs $\bm{z}$, the solution $\bm{p}^*$ in \eqref{eq:entmax} will be a sparse probability vector. Figure \ref{fig:entmax_mappings} depicts the behavior of $\entalpha$-entmax for three different values of $\entalpha$.

\begin{figure}[t]
\centering
\includegraphics[width=0.8\columnwidth]{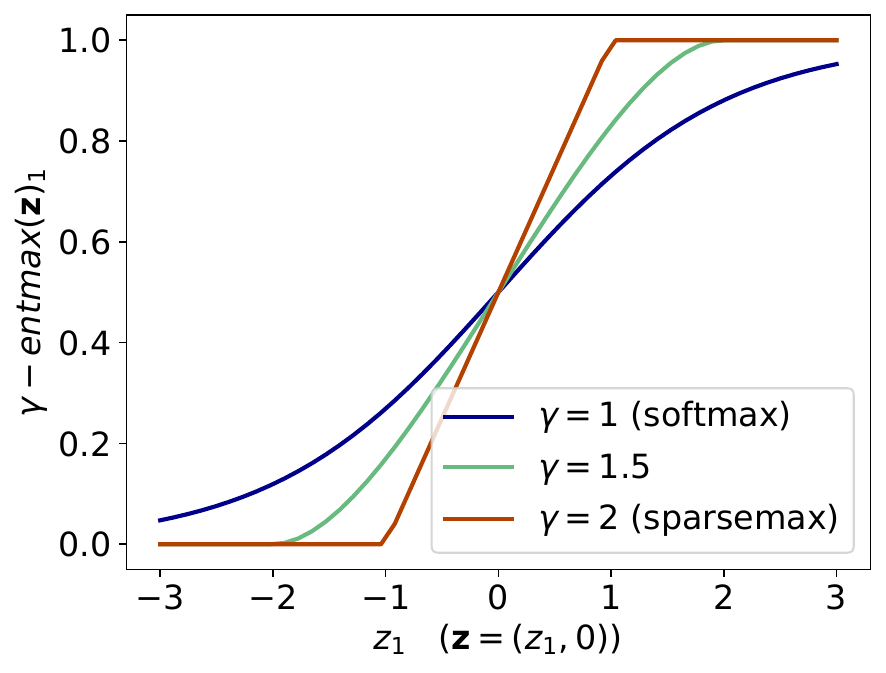}
	\caption{\label{fig:entmax_mappings} Illustration of entmax in the two-dimensional case $\entalpha$-entmax([$t$, 0])$_1$.}
\end{figure}

\citet{peters2019sparsesequencetosequencemodels} show that 
the solution of \eqref{eq:entmax} has the form 
\begin{align}\label{eq:entmax_def} 
	\entalpha\text{-}\entmax(\bm{z}) = \left[ (\entalpha - 1)\bm{z} - \tau \mathbf{1} \right]_{+}^{\frac{1}{\entalpha - 1}},
\end{align}
where $\mathbf{1}$ denotes a vector of ones,  $\left[ x \right]_+ = \max\{x, 0\}$, and $\tau$ is the (unique) constant ensuring normalization. 
A bisection algorithm for computing $\entalpha$-entmax for general $\entalpha$, as well as a more efficient algorithm for $\entalpha=1.5$,  have been proposed by \citet{peters2019sparsesequencetosequencemodels}. 
We make use of these algorithms in our experiments in \S\ref{sect:experiments}. 

\begin{figure*}
\centering
\includegraphics[width=0.9\textwidth]{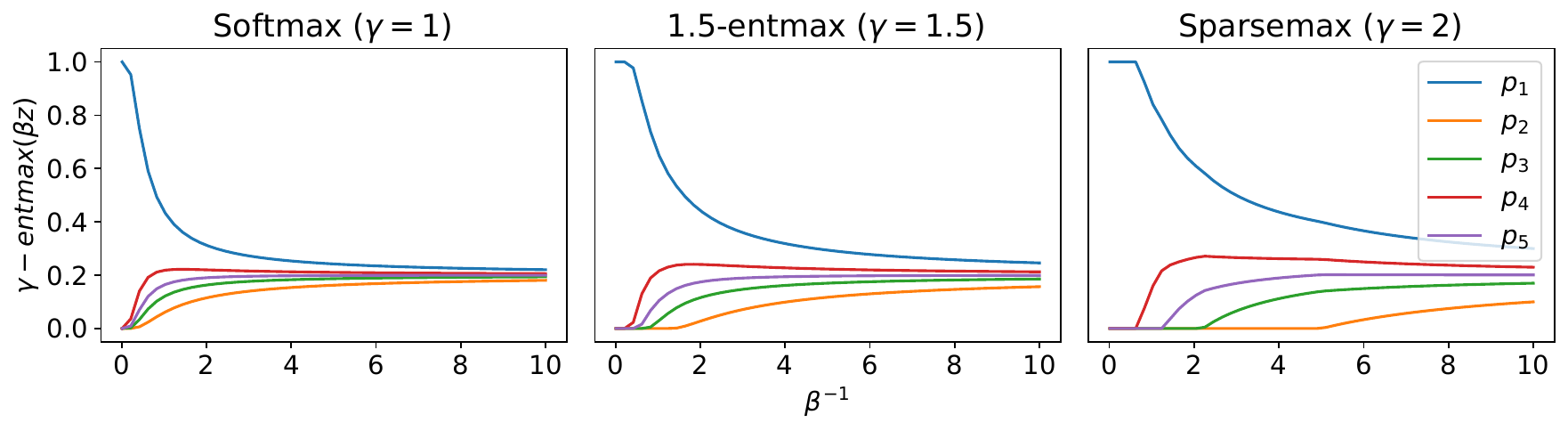}
\caption{Output of different $\entalpha\text{-}\mathsf{entmax}$ transformations on label scores $\bm{z} = [1, -1, -0.2, 0.4, -0.5]$, as a function of temperature parameter $\beta^{-1}$, where $[p_1,p_2,p_3,p_4,p_5] = \entalpha\text{-}\mathsf{entmax}(\beta\bm{z})$.}
\label{fig:entmax_temperature}
\end{figure*} 

\paragraph{Adjusting sparsity} Temperature scaling \citep{guo2017calibrationmodernneuralnetworks,platt1999probabilistic} is a simple and popular technique to calibrate the output probabilities from a predictive model. By multiplying the label scores $\bm{z}$ by a scaling factor $\beta > 0$, which can be interpreted as an inverse temperature parameter, the probabilities resulting from the subsequent softmax transformation can be made more or less peaked. 
When $\beta = 0$, $\softmax(\beta \bm{z})$ becomes a uniform distribution, whereas in the limit $\beta \rightarrow +\infty$ (the ``zero temperature limit''), $\softmax(\beta \bm{z})$ tends to a one-hot vector, assigning probability 1 to the label with the largest score.\footnote{If there are multiple labels tied with the same largest score, $\lim_{\beta\rightarrow +\infty}\softmax(\beta \bm{z})$ becomes a uniform distribution over those labels, assigning zero probability to all other labels.}  

Analogous limit cases apply also to $\entalpha\text{-}\entmax(\beta\bm{z})$, but there is a key qualitative difference \citep{martins2016softmaxsparsemaxsparsemodel,blondel2020learningfenchelyounglosses}. For any $\entalpha>1$, there is a \textit{finite} threshold $\beta^*$ (corresponding to a ``non-zero temperature'') above which $\entalpha\text{-}\entmax(\beta\bm{z})$ saturates as the zero temperature limit: $\beta \geq \beta^* \, \Rightarrow \, \entalpha\text{-}\entmax(\beta\bm{z}) = \lim_{\beta\rightarrow +\infty} \entalpha\text{-}\entmax(\beta\bm{z})$. Increasing $\beta$, which may be seen as a regularization coefficient, from $0$ to $\beta^*$ draws a regularization path where the support of $\entalpha$-$\entmax(\beta \bm{z})$ varies from the full set of labels to a single label (assuming no ties), as shown in Figure~\ref{fig:entmax_temperature}. 
Therefore, using temperature scaling with $\entalpha$-entmax, for $\entalpha > 1$, allows the sparsity of the output to be controlled. 
We exploit this fact in our proposed method described in \S\ref{sec:proposed}. 

\section{\MakeUppercase{Conformalizing Sparse Transformations}}
\label{sec:proposed}

We now present a connection between sparse transformations and conformal prediction, via the design of a non-conformity score for which the prediction sets given by conformal prediction are the same as the support of the sparsemax output (more generally, $\entalpha$-entmax with $\entalpha>1$)  where the calibration step corresponds to setting $\beta$ in temperature scaling. 


\paragraph{Sparsity as set prediction}
Considering that, for $\entalpha > 1$, the $\entalpha$-entmax transformations can produce sparse outputs and that the level of sparsity can be controlled through the coefficient $\beta$, the support of the output can be interpreted as a prediction set for a given input $\bm{x} \in \mathcal{X}$, where $\bm{z}= f(\bm{x}) \in \mathbb{R}^K$ is the vector of label scores produced by a trained model. For notational convenience, we assume throughout that the entries of $\bm{z}$ are sorted in descending order, $z_1 \ge z_2 \ge ... \ge z_K$.

\subsection{Conformalizing sparsemax}\label{subsec:sparsemax_cp}

Let us start with sparsemax. 
Taking a closer look at Algorithm~\ref{alg:sparsemax}, we start by noting that the inequality in line~\ref{alg:sparsemax_line_k} can be equivalently written as\footnote{\rebuttal{By noticing that $1+jz_j > \sum_{k=1}^{j} z_k \Leftrightarrow \left(\sum_{k=1}^{j-1} z_k\right) + z_j -z_j - (j-1)z_j < 1\Leftrightarrow \sum_{k=1}^{j-1} (z_{k} - z_{j}) < 1$.}} $\sum_{k=1}^{j-1} (z_{k} - z_{j}) < 1$. Let $S(\bm{z}) := \{j \in [K] : \mathsf{sparsemax}_j(\bm{z}) > 0\}$ denote the support of $\mathsf{sparsemax}(\bm{z})$, \textit{i.e.}, the labels that are assigned strictly positive probability. 
Scaling the label scores by the coefficient $\beta > 0$, we obtain a necessary and sufficient condition for the $j\textsuperscript{th}$ highest ranked label to be in the support $S(\beta \bm{z})$:
\begin{align}\label{eq:support_sparsemax}
    j \in S(\beta \bm{z}) \iff \sum_{k=1}^{j-1} (z_k - z_j) < \beta^{-1}.
\end{align}
This fact leads to  the following proposition.
\begin{proposition}\label{prop:sparsemax_conformal}
Let $C_\alpha: \mathcal{X}\rightarrow 2^\mathcal{Y}$ be a conformal predictor (as described in \S\ref{subsec:background_cp}). 
Define the following nonconformity score:
\begin{align}\label{eq:nonconf_score_sparsemax}
s(x,y) = \sum_{k=1}^{k(y)-1} (z_k - z_{k(y)}),
\end{align}
where $k(y)$ is the index of label $y$ in the sorted array $\bm{z}$, 
and let $\hat{q}$ be the $\lceil(n+1)(1-\alpha)\rceil/n$ empirical quantile of the set of calibration scores. 
Then, setting the sparsemax temperature as $\beta^{-1} := \hat{q}$ at test time leads to prediction sets $C_\alpha(x) = S(\beta\bm{z})$ achieving the desired $(1-\alpha)$ coverage in expectation. 
\end{proposition}

\begin{proof}
    This follows directly from the coverage guarantee of conformal prediction \eqref{eq:coverage} and condition \eqref{eq:support_sparsemax}. 
\end{proof}

The non-conformity score \eqref{eq:nonconf_score_sparsemax} has an intuitive interpretation: it sums the score differences between each of the top ranked labels and the label $y$. The lower the rank of $y$, the larger the number of terms in this sum  will be; conversely, the lower the score of $y$ is compared to the scores of the top labels, the larger the total sum will be. If this sum is larger than a threshold, the label will not be included in the prediction set. 

\subsection{Conformalizing $\gamma$-entmax} \label{subsec:entmax_cp}
We next turn to the more general case of $\entalpha$-entmax. 
Let $S(\beta \bm{z}; \entalpha) := \{j \in [K] : \entalpha\text{-entmax}_j(\beta\bm{z})>0\}$ denote the support of the distribution induced by the $\entalpha$-entmax transformation with  temperature $\beta^{-1}$ on the score vector $\bm{z}$. 
We start with the following result, which generalizes \eqref{eq:support_sparsemax}.

\begin{proposition}\label{prop:tau}
	The following condition characterizes the set of labels in the support $S(\beta\bm{z}; \entalpha)$: 
	\begin{equation}\label{eq:support_conditon}
j \in S(\beta \bm{z}; \entalpha) \iff \sum_{k=1}^{j-1} \big[(\entalpha-1)\beta (z_{k}-z_{j})\big]^{\frac{1}{\entalpha-1}} < 1.
	\end{equation}
\end{proposition}
\begin{proof}
	The proof can be found in Appendix~\ref{sec:proof_tau}.
\end{proof}

Defining $\delta = 1/(\entalpha - 1)$ and letting $\|\bm{v}\|_\delta := (\sum_{k=1}^{K} |v_k|^\delta)^\frac{1}{\delta}$ be the $\delta$-norm in $\mathbb{R}^K$, we can equivalently express \eqref{eq:support_conditon} as 
\begin{equation} \label{eq:support_norm_eq}
    j \in S(\beta\bm{z}, \entalpha) \iff \|\bm{z}_{1:(j-1)} - z_{j}\bm{1}\|_\delta < \delta\beta^{-1}. 
\end{equation}
The sparsemax case \eqref{eq:support_sparsemax} is recovered by setting $\entalpha=2$ (equivalently, $\delta=1$) in  \eqref{eq:support_norm_eq}. 
We then have the following result, generalizing Proposition~\ref{prop:sparsemax_conformal}:

	
	

\begin{proposition}\label{prop:entmax_conformal}
Let $C_\alpha: \mathcal{X}\rightarrow 2^\mathcal{Y}$, be a conformal predictor (as described in  \S\ref{subsec:background_cp}). 
Define the following nonconformity score:
\begin{align}\label{eq:nonconf_score_entmax}
s(x,y) = \|\bm{z}_{1:k(y)}-z_{k(y)}\mathbf{1}\|_\delta,
\end{align}
where $\hat{q}$ is the $\lceil(n+1)(1-\alpha)\rceil/n$ empirical quantile of the set of calibration scores. 
Then, setting the $\entalpha$-entmax temperature as $\beta^{-1} := \delta^{-1}\hat{q}$ at test time leads to prediction sets $C_\alpha(x) = S(\beta\bm{z}; \entalpha)$ achieving the desired $(1-\alpha)$ coverage in expectation. 
\end{proposition}

\begin{proof}
    This again follows directly from the conformal prediction coverage guarantee \eqref{eq:coverage} and condition \eqref{eq:support_norm_eq}. 
\end{proof}

Proposition \ref{prop:entmax_conformal} states that, through an appropriate choice of non-conformity score, a conformal predictor generates prediction sets corresponding to the support of the $\entalpha$-entmax transformation with temperature $\beta^{-1} = \delta^{-1}\hat{q}$, \rebuttal{as illustrated in Figure \ref{fig:diagram}}. 
Intuitively, the nonconformity score accumulates the differences between the largest label scores and the score of the correct label, returning the $\delta$-norm of the vector of differences. For example, when $\gamma=1.5$ (1.5-entmax), we have $\delta=2$, which corresponds to the Euclidean norm. The non-conformity score \eqref{eq:nonconf_score_entmax} is always non-negative and it is zero if $k(y)=1$. 

\paragraph{Log-margin} 
Interestingly, the non-conformity score \eqref{eq:nonconf_score_entmax} also works for softmax ($\entalpha=1$, \textit{i.e.}, $\delta = +\infty$):
\begin{align}\label{eq:nonconf_score_softmax_new}
    s(x,y) &= \|\bm{z}_{1:k(y)}-z_{k(y)}\mathbf{1}\|_\infty = z_1 - z_{k(y)} \nonumber\\
    &= \log \frac{p_1}{p_{k(y)}},
\end{align}
which is the log-odds ratio between the most probable class and the true one.
Since the $\log$ function is monotonic, calibration of this non-conformity score leads to thresholding the odds ratio $p_1/ p_{k(y)}$, which has an intuitive interpretation: labels whose probability is above a fraction of that of the most probable label are included in the prediction set.\footnote{This non-conformity score, although reminiscent of the previously used \textit{margin score} \citep{johansson_marginscores,linusson_margin}, uses the margin (distance to the maximum) on model output, not on the probability estimates.} However, the interpretation as temperature calibration no longer applies, since softmax is not sparse. We also experiment with the non-conformity score \eqref{eq:nonconf_score_softmax_new} in \S\ref{sect:experiments}. 

\section{\MakeUppercase{Experiments}}
\label{sect:experiments}

To assess the performance of the non-conformity scores introduced in \S\ref{sec:proposed} we report experiments on several classification tasks, comparing the proposed strategies with standard conformal prediction  over several dimensions: coverage, efficiency (prediction set size), and adaptiveness, at several different confidence levels.  

\subsection{Experimental Setup}
 \paragraph{Datasets} We evaluate all approaches on tasks of varying difficulty: image classification on the CIFAR10, CIFAR100, and ImageNet datasets \citep{cifars,imagenet} and text classification on the 20 Newsgroups dataset \citep{twenty_newsgroups_113}. For each dataset, the original test data is split into calibration ($40\%$) and test sets ($60\%$) and the results reported are averaged over 5 random splits.

\paragraph{Models} For the CIFAR100 and ImageNet datasets, we finetune the \textit{vision transformer} (ViT) model \citep{dosovitskiy2021imageworth16x16words}, obtaining average accuracies of approximately $0.86$ and $0.81$ on the test set, respectively. As for the 20 Newsgroups dataset, we finetune a BERT base model \citep{Devlin2019BERTPO} for sequence classification, obtaining an average test accuracy of $0.74$. Given the simplicity of the task on the CIFAR10 dataset, we train a convolutional neural network from scratch with final average test accuracy of $0.84$. More model evaluation details can be found in Appendix \ref{subsec:modelling}.

\begin{table*}[t]
\centering
\caption{Empirical coverage of different conformal procedures on the test set (averaged over 5 splits).} 
\smallskip
\begin{tabular}{l|l|cccccc}
\toprule
$\mathbf{\alpha}$ & \textbf{Dataset} & $\mathsf{RAPS}$ & $1.5$-$\mathsf{entmax}$ & $\mathsf{log\text{-}margin}$ & $\mathsf{opt\text{-}entmax}$ & $\mathsf{InvProb}$ & $\mathsf{sparsemax}$ \\
\midrule
\multirow{4}{*}{0.01} & CIFAR10    & 0.990 & 0.990 & 0.990 & 0.989 & 0.990 & 0.989 \\
                      & CIFAR100   & 0.991 & 0.991 & 0.991 & 0.991 & 0.990 & 0.991 \\
                      & ImageNet   & 0.990 & 0.990 & 0.990 & 0.990 & 0.990 & 0.990 \\
                      & NewsGroups & 0.989 & 0.988 & 0.989 & 0.990 & 0.989 & 0.989 \\
\midrule
\multirow{4}{*}{0.10} & CIFAR10    & 0.899 & 0.900 & 0.901 & 0.900 & 0.900 & 0.900 \\
                      & CIFAR100   & 0.897 & 0.899 & 0.900 & 0.899 & 0.899 & 0.900 \\
                      & ImageNet   & 0.899 & 0.899 & 0.899 & 0.899 & 0.899 & 0.900 \\
                      & NewsGroups & 0.895 & 0.900 & 0.902 & 0.900 & 0.901 & 0.900 \\
\bottomrule
\end{tabular}
\label{tab:coverage}
\end{table*}

\paragraph{Conformal procedures} We experiment with two commonly used conformal procedures, both involving the softmax transformation:
\begin{itemize}
    \item $\mathsf{InvProb}$: this is the standard conformal prediction (\S \ref{subsec:background_cp}) with non-conformity score $s(x,y) = 1 - \hat{p}(y|x)$, where $\hat{p}(y|x)$ is the softmax output corresponding to class $y$;
    \item $\mathsf{RAPS}$: this is a modification of the standard procedure called \textit{regularized adaptive prediction sets} \citep{angelopoulos2022uncertaintysetsimageclassifiers}, which makes use of two regularization parameters to improve the efficiency of the \textit{adaptive predictive sets} (APS) predictors \citep{romano2020classificationvalidadaptivecoverage}. In prior work, this has shown good efficiency and adaptiveness properties. \rebuttal{Implementation of the $\mathsf{RAPS}$ procedure is done following the original paper\footnote{Since using the same set to both calibrate the conformal predictor and find the optimal regularization hyperparameters would violate the exchangeability assumption, we split the original calibration data into two subsets for this procedure.}} (additional details can be found in Appendix \ref{subsec:raps}).
\end{itemize}
We compare these approaches with the following proposed ones: 
\begin{itemize}
    \item $\entalpha\text{-}\mathsf{entmax}$: this applies conformal prediction with the non-conformity score (\ref{eq:nonconf_score_entmax}) as described in Proposition \ref{prop:entmax_conformal}, with $1<\gamma\leq 2$, which has a temperature scaling interpretation. When $\gamma=2$, we call the procedure $\mathsf{sparsemax}$, corresponding to non-conformity score (\ref{eq:nonconf_score_sparsemax});
    \item $\mathsf{log\text{-}margin}$: this uses the score defined in (\ref{eq:nonconf_score_softmax_new}), which corresponds to $\gamma=1$ (softmax), also a particular case of the non-conformity score (\ref{eq:nonconf_score_entmax}), but without an interpretation in terms of temperature scaling;
    \item $\mathsf{opt\text{-}entmax}$: a procedure where $\gamma$ is treated as a hyperparameter ($1<\gamma<2$), tuned to minimize the average prediction set size. \rebuttal{This is done by splitting the original calibration set in two: one for conformal prediction calibration and the second for the choice of optimal $\gamma$. Additional setup details can be found in Appendix \ref{subsec:opt_entmax}.}
\end{itemize}

\begin{figure*}[t]
\centering
\includegraphics[width=\textwidth]{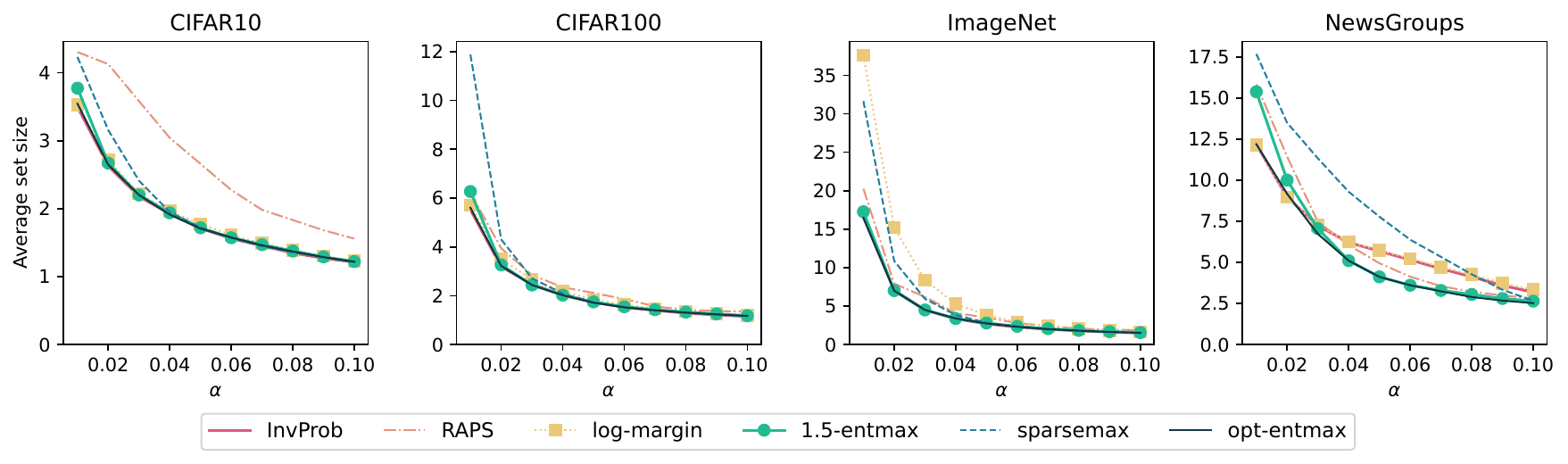}
\caption{Average prediction set size as a function of significance level $\alpha$. \label{fig:avg_set_size_all}}
\end{figure*}

\subsection{Coverage and Efficiency}
Table \ref{tab:coverage} shows the average empirical coverage obtained by the different methods on the test set for two confidence levels $1-\alpha$, with $\alpha \in \{0.01, 0.1\}$. 
We observe that all methods achieve a marginal coverage close to the theoretical bound of $1-\alpha$, as expected (see Appendix \ref{subsec:extra_coverage} for coverage analysis for more values of $\alpha$).

In order to evaluate the efficiency of the predictors, we use two common metrics: the \textit{average set size} and the \textit{singleton ratio} (fraction of prediction sets containing a single element). 
We vary $\alpha \in [0.01, 0.1]$ to study different confidence levels. 

\paragraph{Average set size} As can be seen in Figure \ref{fig:avg_set_size_all}, $\optentmax$ and $\mathsf{InvProb}$ are the most efficient set predictors across almost all tasks and confidence levels (with $\optentmax$ superior to $\mathsf{InvProb}$ in NewsGroups), both dominating $\mathsf{RAPS}$. The 1.5-$\entmax$ is in general very competitive, and so is the $\mathsf{log\text{-}margin}$ predictor (except for ImageNet). The $\sparsemax$ predictor  generally yields larger sets on average, especially for high confidence levels. 

In $\entalpha\text{-}\entmax$, $\gamma$ can be seen as an additional hyperparameter, since  for each $\alpha$, it can be tuned on the calibration set for a given metric, 
as $\optentmax$ does for average set size. An example of the behavior of $\entalpha\text{-}\entmax$ with varying $\gamma$ is shown in Figure \ref{fig:gamma_entmax}.

\begin{figure}[t]
    \centering
    \includegraphics[width=0.9\linewidth]{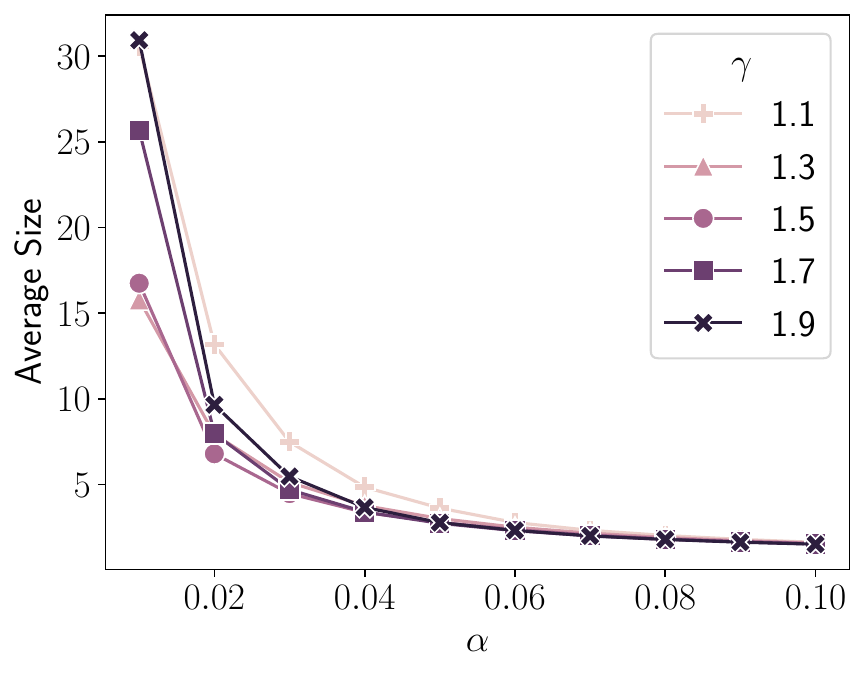}
    \caption{Average set size as a function of $\alpha$, for the ImageNet dataset with varying $\gamma$ for $\entmax$.}
    \label{fig:gamma_entmax}
\end{figure}

\paragraph{Singleton ratio} Considering the ratio of singleton predictions, shown in Figure \ref{fig:singleton_ratio}, we observe significant differences between methods: $\raps$ and $\logmargin$ have the lowest  and highest ratio, respectively, across all tasks and confidence levels; $\invprob$, $\logmargin$ and $\optentmax$ output singletons even for $\alpha=0.01$ for all tasks except for 20 Newsgroups, where no method produces singletons, which could be related to the lower accuracy of the original predictive model. The coverage of singleton predictions, also shown in Figure \ref{fig:singleton_ratio}, is approximately $1-\alpha$, as desired, for all methods.

\begin{figure*}[t]
\centering
\includegraphics[width=0.95\textwidth]{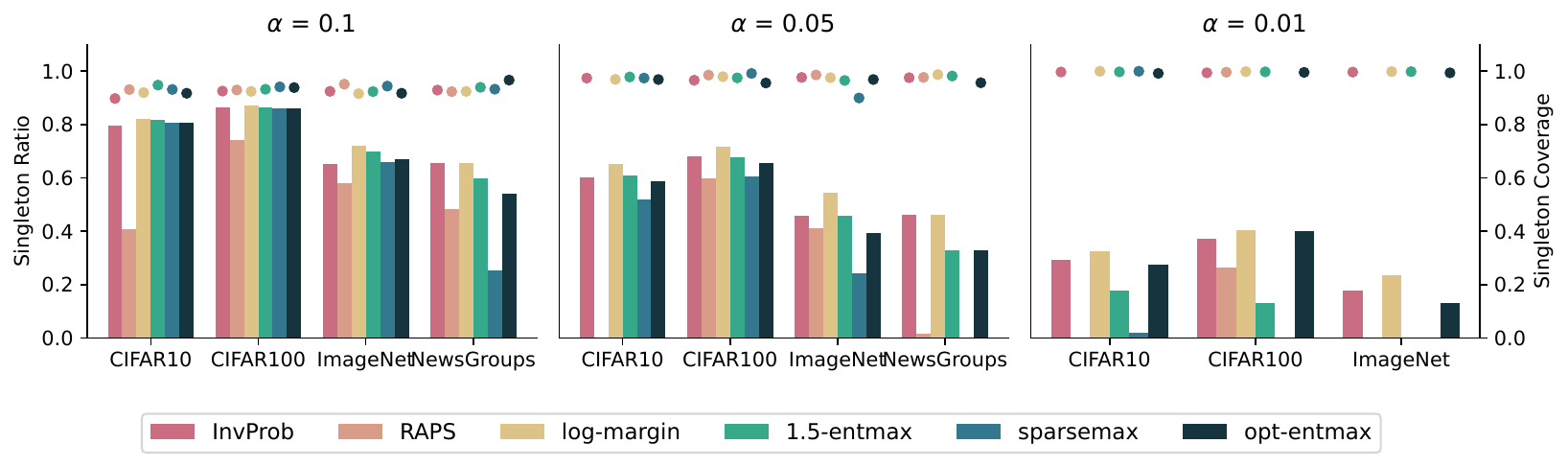}
\caption{Ratio of singleton prediction sets (bars) and their coverage (points) for different values of confidence level $1-\alpha$ (NewsGroups is absent for $\alpha=0.01$ because no singletons were predicted by either method).\label{fig:singleton_ratio}}
\end{figure*}

\subsection{Adaptiveness} Given that conformal predictors are unable to provide \textit{conditional} coverage theoretical guarantees \citep{pmlr-v25-vovk12}, \textit{i.e.}, there is no lower bound on $\mathbb{P}\big( Y_\text{test} \in \mathcal{C}_\alpha(X_\text{test})|X_\text{test} \big)$ without extra assumptions, different measures have been proposed to calculate proximity to conditional coverage. Ideally, for any given partition of the data, we would have a coverage close to the $1-\alpha$ bound. 
\citet{angelopoulos2022uncertaintysetsimageclassifiers} introduced an adaptiveness criterion based on \textit{size-stratified coverage}, by evaluating coverage in a partition based on the size of the predicted sets. Table \ref{tab:imagenet_size_coverage} shows the size-stratified coverage for the ImageNet dataset, with $\alpha=0.01$ and $\alpha=0.1$. For $\alpha=0.1$, $\logmargin$ stands out from other methods, achieving a lower deviation from exact coverage on all set cardinalities. Size-stratified coverage improves for lower values of $\alpha$, and $\invprob$, $\logmargin$ and $\optentmax$ exhibit the desired behavior---achieving expected coverage while still predicting small sets for easier examples. 

\begin{table*}
\centering
\caption{Size-stratified coverage for the ImageNet dataset with $\alpha=0.01$ (top) and $\alpha=0.1$ (bottom).}
\smallskip 

\begin{tabular}{lrrrrrrrrrrrr}
\toprule
\multirow{2}{*}{\textbf{size}} & \multicolumn{2}{c}{\textbf{InvProb}} & \multicolumn{2}{c}{\textbf{RAPS}} & \multicolumn{2}{c}{\textbf{log-margin}} & \multicolumn{2}{c}{\textbf{1.5-entmax}} & \multicolumn{2}{c}{\textbf{sparsemax}} & \multicolumn{2}{c}{\textbf{opt-entmax}} \\
\cmidrule(lr){2-3} \cmidrule(lr){4-5} \cmidrule(lr){6-7} \cmidrule(lr){8-9} \cmidrule(lr){10-11} \cmidrule(lr){12-13}
 & n & cov & n & cov & n & cov & n & cov & n & cov & n & cov \\
\midrule
0 - 1 &    5339 &  0.994 &      0 &   -- &       7088 &  0.992 &         44 &  1.000 &         0 &    -- &       3899 &  0.996 \\
2 - 3 &    7665 &  0.992 &      0 &   -- &       7970 &  0.989 &       1559 &  1.000 &         0 &    -- &       7383 &  0.994 \\
4 - 6 &    5403 &  0.990 &      0 &   -- &       4821 &  0.986 &       5196 &  0.998 &         0 &    -- &     6159 &  0.991 \\
7 -10 &    3221 &  0.990 &      0 &   -- &       2604 &  0.982 &       7585 &  0.997 &       177 &  1.000 &     3814 &  0.992 \\
11 - 1000  &    8372 &  0.981 &  30000 &  0.99 &       7517 &  0.987 &      15616 &  0.982 &     29823 &  0.989 &       8745 &  0.978 \\
\bottomrule
\end{tabular}

\begin{tabular}{lrrrrrrrrrrrrrrr}
\toprule
\multirow{2}{*}{\textbf{size}} & \multicolumn{2}{c}{\textbf{InvProb}} & \multicolumn{2}{c}{\textbf{RAPS}} & \multicolumn{2}{c}{\textbf{log-margin}} & \multicolumn{2}{c}{\textbf{1.5-entmax}} & \multicolumn{2}{c}{\textbf{sparsemax}} & \multicolumn{2}{c}{\textbf{opt-entmax}} \\
\cmidrule(lr){2-3} \cmidrule(lr){4-5} \cmidrule(lr){6-7} \cmidrule(lr){8-9} \cmidrule(lr){10-11} \cmidrule(lr){12-13}
 & n & cov & n & cov & n & cov & n & cov & n & cov & n & cov \\
\midrule
0 - 1   & 19610 & 0.932 & 17408 & 0.948 & 21664 & 0.924 & 20939 & 0.933 & 19820 & 0.942 & 20074 & 0.940 \\
2 - 3   & 9321  & 0.851 & 8862  & 0.888 & 6346  & 0.859 & 7358  & 0.849 & 9098  & 0.838 & 8731  & 0.839 \\
4 - 6   & 917   & 0.808 & 2696  & 0.774 & 1404  & 0.767 & 1407  & 0.728 & 1048  & 0.618 & 1140  & 0.636 \\
7 - 10  & 4     & 0.500 & 753   & 0.616 & 373   & 0.743 & 262   & 0.588 & 34    & 0.6 & 50    & 0.519 \\
11 - 1000 & 0   &  --  & 0     &  --  & 213   & 0.653 & 34    & 0.559 & 0     &  --  & 1     &  0  \\
\bottomrule
\end{tabular}
\label{tab:imagenet_size_coverage}
\end{table*}

\section{\MakeUppercase{Discussion}}
\label{sec:discussion}

\paragraph{Choosing a score} The choice of non-conformity score plays a crucial role in the efficiency of a conformal predictor \citep{aleksandrova21a_impact}; consequently, finding useful model-agnostic scores is key for the adoption of conformal prediction as an actionable uncertainty quantification framework. In \S\ref{sec:proposed}, we introduced a new family of non-conformity measures that we show to be competitive with standard scores on different classification tasks (\S \ref{sect:experiments}). Expanding the range of possible non-conformity measures is particularly relevant, considering that the optimal score choice is known to be task-dependent \citep{laxhammar}. Additionally, we show that these scores can be used to tune the adaptiveness of a conformal predictor.

\paragraph{Temperature scaling connection} 
Conformal prediction and temperature scaling are two popular, yet very distinct, approaches to the task of model calibration and confidence estimation. The fact that there is an equivalence between the two methods for a family of activation functions is not only intrinsically interesting but also promising, as it opens some research threads: designing new non-conformity scores from sparse transformations, or comparing an activation natural calibration capacity through the evaluation of the correspondent conformal predictor.

\section{\MakeUppercase{Conclusion}}
\label{sec:conclusion}
In this paper, we showed a novel connection between conformal prediction and temperature scaling in sparse activation functions. Specifically, we derived  new non-conformity scores which make conformal set prediction equivalent to temperature scaling of the $\entalpha\text{-}\mathsf{entmax}$ family of sparse activations. This connection allows $\entalpha\text{-}\mathsf{entmax}$ with the proposed calibration procedure to inherit the strong coverage guarantees of conformal prediction. Experiments with computer vision and text classification benchmarks show the efficiency and adaptiveness of the proposed non-conformity scores, showing their practical usefulness. 
\section{\MakeUppercase{Acknowledgements}}
\label{sec:acknowledgement}
\rebuttal{
This work was supported by the Portuguese Recovery
and Resilience Plan through project C645008882-
00000055 (NextGenAI - Center for Responsible AI), by the EU's
Horizon Europe Research and Innovation Actions
(UTTER, contract 101070631), by the project
DECOLLAGE (ERC-2022-CoG 101088763), 
and by FCT/MECI through national funds and when applicable co-funded EU funds under UID/50008: Instituto de Telecomunicações. 
}
\bibliography{references}

\begin{thebibliography}{24}
\providecommand{\natexlab}[1]{#1}
\providecommand{\url}[1]{\texttt{#1}}
\expandafter\ifx\csname urlstyle\endcsname\relax
  \providecommand{\doi}[1]{doi: #1}\else
  \providecommand{\doi}{doi: \begingroup \urlstyle{rm}\Url}\fi

\bibitem[Silva-Filho et~al.(2023)Silva-Filho, Song, Perello-Nieto, Santos-Rodriguez, Kull, and Flach]{Silva_Filho}
T.~Silva-Filho, H.~Song, M.~Perello-Nieto, R.~Santos-Rodriguez, M.~Kull, and P.~Flach.
\newblock Classifier calibration: a survey on how to assess and improve predicted class probabilities.
\newblock \emph{Machine Learning}, 112:\penalty0 3211--–3260, 2023.
\newblock URL \url{https://doi.org/10.1007/s10994-023-06336-7}.

\bibitem[Gawlikowski et~al.(2023)Gawlikowski, Tassi, and et~al.]{Gawlikowski}
J.~Gawlikowski, C.~Tassi, and M.~Ali et~al.
\newblock A survey of uncertainty in deep neural networks.
\newblock \emph{Artificial Intelligence Review}, 56:\penalty0 1513--1589, 2023.
\newblock URL \url{https://doi.org/10.1007/s10462-023-10562-9}.

\bibitem[Guo et~al.(2017)Guo, Pleiss, Sun, and Weinberger]{guo2017calibrationmodernneuralnetworks}
Chuan Guo, Geoff Pleiss, Yu~Sun, and Kilian~Q. Weinberger.
\newblock On calibration of modern neural networks.
\newblock In \emph{Proceedings of the Internation Conference on Machine Learning (ICML)}, 2017.

\bibitem[Wenger et~al.(2020)Wenger, Kjellström, and Triebel]{wenger2020}
Jonathan Wenger, Hedvig Kjellström, and Rudolph Triebel.
\newblock Non-parametric calibration for classification.
\newblock In \emph{Proceedings of the 23rd International Conference on Artificial Intelligence and Statistics (AISTATS)}, 2020.

\bibitem[Vovk et~al.(2005)Vovk, Gammerman, and Shafer]{vovk_book}
V.~Vovk, A.~Gammerman, and G.~Shafer.
\newblock \emph{Algorithmic Learning in a Random World}.
\newblock Springer, 2005.

\bibitem[Angelopoulos and Bates(2023)]{angelopoulos2022gentleintroductionconformalprediction}
Anastasios~N. Angelopoulos and Stephen Bates.
\newblock A gentle introduction to conformal prediction and distribution-free uncertainty quantification.
\newblock \emph{Foundations and Trends in Machine Learning}, 16\penalty0 (4):\penalty0 494--591, 2023.

\bibitem[Papadopoulos et~al.(2002)Papadopoulos, Proedrou, Vovk, and Gammerman]{papadopoulos2002inductive}
Harris Papadopoulos, Kostas Proedrou, Volodya Vovk, and Alex Gammerman.
\newblock Inductive confidence machines for regression.
\newblock In \emph{13th European Conference on Machine Learning (ECML)}, pages 345--356, 2002.

\bibitem[Martins and Astudillo(2016)]{martins2016softmaxsparsemaxsparsemodel}
Andre Martins and Ramon Astudillo.
\newblock From softmax to sparsemax: A sparse model of attention and multi-label classification.
\newblock In M.~Balcan and K.~Weinberger, editors, \emph{Proceedings of The 33rd International Conference on Machine Learning (ICML)}, volume~48 of \emph{Proceedings of Machine Learning Research}, pages 1614--1623, 20--22 Jun 2016.
\newblock URL \url{https://proceedings.mlr.press/v48/martins16.html}.

\bibitem[Peters et~al.(2019)Peters, Niculae, and Martins]{peters2019sparsesequencetosequencemodels}
Ben Peters, Vlad Niculae, and Andr{\'e} F.~T. Martins.
\newblock Sparse sequence-to-sequence models.
\newblock In Anna Korhonen, David Traum, and Llu{\'i}s M{\`a}rquez, editors, \emph{Proceedings of the 57th Annual Meeting of the Association for Computational Linguistics (ACL)}, pages 1504--1519, 2019.

\bibitem[Blondel et~al.(2020)Blondel, Martins, and Niculae]{blondel2020learningfenchelyounglosses}
Mathieu Blondel, André F.~T. Martins, and Vlad Niculae.
\newblock Learning with {Fenchel-Young} losses.
\newblock \emph{Journal of Machine Leaning Research}, 21:\penalty0 1--69, 2020.

\bibitem[Tsallis(1988)]{tsallis1988}
Constantino Tsallis.
\newblock Possible generalization of boltzmann-gibbs statistics.
\newblock \emph{Journal of Statistical Physics}, 52:\penalty0 479--487, 07 1988.
\newblock \doi{10.1007/BF01016429}.

\bibitem[Platt et~al.(1999)]{platt1999probabilistic}
John Platt et~al.
\newblock Probabilistic outputs for support vector machines and comparisons to regularized likelihood methods.
\newblock In \emph{Advances in Large Margin Classifiers}, pages 61--74. MIT Press, 1999.

\bibitem[Johansson et~al.(2017)Johansson, Linusson, Löfström, and Boström]{johansson_marginscores}
Ulf Johansson, Henrik Linusson, Tuve Löfström, and Henrik Boström.
\newblock Model-agnostic nonconformity functions for conformal classification.
\newblock In \emph{2017 International Joint Conference on Neural Networks (IJCNN)}, pages 2072--2079, 2017.
\newblock \doi{10.1109/IJCNN.2017.7966105}.

\bibitem[Linusson et~al.(2018)Linusson, Johansson, Bostr{\"o}m, and L{\"o}fstr{\"o}m]{linusson_margin}
Henrik Linusson, Ulf Johansson, Henrik Bostr{\"o}m, and Tuve L{\"o}fstr{\"o}m.
\newblock Classification with reject option using conformal prediction.
\newblock In Dinh Phung, V.~Tseng, G.~Webb, B.~Ho, M.~Ganji, and L.~Rashidi, editors, \emph{Advances in Knowledge Discovery and Data Mining}, pages 94--105, 2018.

\bibitem[Krizhevsky(2009)]{cifars}
Alex Krizhevsky.
\newblock Learning multiple layers of features from tiny images.
\newblock Technical report, University of Toronto, 2009.
\newblock URL \url{https://api.semanticscholar.org/CorpusID:18268744}.

\bibitem[Deng et~al.(2009)Deng, Dong, Socher, Li, Li, and Fei-Fei]{imagenet}
Jia Deng, Wei Dong, Richard Socher, Li-Jia Li, Kai Li, and Li~Fei-Fei.
\newblock Imagenet: A large-scale hierarchical image database.
\newblock In \emph{2009 IEEE Conference on Computer Vision and Pattern Recognition}, pages 248--255, 2009.
\newblock \doi{10.1109/CVPR.2009.5206848}.

\bibitem[Mitchell(1997)]{twenty_newsgroups_113}
Tom Mitchell.
\newblock {Twenty Newsgroups}.
\newblock UCI Machine Learning Repository, 1997.
\newblock {DOI}: https://doi.org/10.24432/C5C323.

\bibitem[Kolesnikov et~al.(2021)Kolesnikov, Dosovitskiy, Weissenborn, Heigold, Uszkoreit, Beyer, Minderer, Dehghani, Houlsby, Gelly, Unterthiner, and Zhai]{dosovitskiy2021imageworth16x16words}
Alexander Kolesnikov, Alexey Dosovitskiy, Dirk Weissenborn, Georg Heigold, Jakob Uszkoreit, Lucas Beyer, Matthias Minderer, Mostafa Dehghani, Neil Houlsby, Sylvain Gelly, Thomas Unterthiner, and Xiaohua Zhai.
\newblock An image is worth 16x16 words: Transformers for image recognition at scale.
\newblock In \emph{International Conference on Learning Representations (ICLR)}, 2021.

\bibitem[Devlin et~al.(2019)Devlin, Chang, Lee, and Toutanova]{Devlin2019BERTPO}
Jacob Devlin, Ming-Wei Chang, Kenton Lee, and Kristina Toutanova.
\newblock {BERT}: Pre-training of deep bidirectional transformers for language understanding.
\newblock In Jill Burstein, Christy Doran, and Thamar Solorio, editors, \emph{Proceedings of the 2019 Conference of the North {A}merican Chapter of the Association for Computational Linguistics (ACL)}, pages 4171--4186, 2019.

\bibitem[Angelopoulos et~al.(2021)Angelopoulos, Bates, Jordan, and Malik]{angelopoulos2022uncertaintysetsimageclassifiers}
Anastasios~Nikolas Angelopoulos, Stephen Bates, Michael Jordan, and Jitendra Malik.
\newblock Uncertainty sets for image classifiers using conformal prediction.
\newblock In \emph{International Conference on Learning Representations (ICLR)}, 2021.

\bibitem[Romano et~al.(2020)Romano, Sesia, and Cand\`{e}s]{romano2020classificationvalidadaptivecoverage}
Yaniv Romano, Matteo Sesia, and Emmanuel~J. Cand\`{e}s.
\newblock Classification with valid and adaptive coverage.
\newblock In \emph{Proceedings of the 34th International Conference on Neural Information Processing Systems (NeurIPS)}, 2020.

\bibitem[Vovk(2012)]{pmlr-v25-vovk12}
Vladimir Vovk.
\newblock Conditional validity of inductive conformal predictors.
\newblock In Steven C.~H. Hoi and Wray Buntine, editors, \emph{Proceedings of the Asian Conference on Machine Learning}, volume~25 of \emph{Proceedings of Machine Learning Research}, pages 475--490, 2012.
\newblock URL \url{https://proceedings.mlr.press/v25/vovk12.html}.

\bibitem[Aleksandrova and Chertov(2021)]{aleksandrova21a_impact}
Marharyta Aleksandrova and Oleg Chertov.
\newblock Impact of model-agnostic nonconformity functions on efficiency of conformal classifiers: an extensive study.
\newblock In Lars Carlsson, Zhiyuan Luo, Giovanni Cherubin, and Khuong An~Nguyen, editors, \emph{Proceedings of the Tenth Symposium on Conformal and Probabilistic Prediction and Applications}, volume 152 of \emph{Proceedings of Machine Learning Research}, pages 151--170, 08--10 Sep 2021.
\newblock URL \url{https://proceedings.mlr.press/v152/aleksandrova21a.html}.

\bibitem[Laxhammar and Falkman(2010)]{laxhammar}
Rikard Laxhammar and G\"{o}ran Falkman.
\newblock Conformal prediction for distribution-independent anomaly detection in streaming vessel data.
\newblock In \emph{Proceedings of the First International Workshop on Novel Data Stream Pattern Mining Techniques}, page 47–55, 2010.

\end{thebibliography}

\bibliographystyle{unsrtnat}

\clearpage

\appendix
\onecolumn

\section{Proof of Proposition~\ref{prop:tau}}\label{sec:proof_tau}

Let $\pi$ be the permutation over $[K]$ such that $z_{\pi(1)} \ge z_{\pi(2)} \ge ... \ge z_{\pi(K)}$ are the entries of $\bm{z}$ sorted in ascending order. 
We have $\entalpha\text{-entmax}_j(\beta\bm{z}) = \left[ (\entalpha - 1)\beta z_j - \tau \right]_{+}^{\frac{1}{\entalpha - 1}}$, where $\tau$ satisfies 
\begin{align}\label{eq:proof_equality}
    \sum_{k = 1}^{|S(\beta \bm{z})|} [(\entalpha-1) \beta z_{\pi(k)} - \tau]^{\frac{1}{\entalpha - 1}} = 1.
\end{align}

We prove the desired equivalence by showing implication in both directions, as follows. 

\begin{itemize}
    \item \framebox{$\pi(j) \in S(\beta \bm{z}) \Longrightarrow \sum_{k=1}^{j-1} \left[ (\entalpha - 1)\beta (z_{\pi(k)} - z_{\pi(j)}) \right]^{\frac{1}{\entalpha - 1}} < 1$}

To prove this direction, note that, if $\pi(j) \in S(\beta \bm{z})$, we must have by \eqref{eq:entmax_def} that $(\entalpha - 1) \beta z_{\pi(j)} > \tau$. 
Therefore, we have that, for all $k \in [K]$, $(\entalpha - 1) \beta (z_{\pi(k)} - z_{\pi(j)}) < (\entalpha - 1) \beta z_{\pi(k)} - \tau$. 
As a consequence, 
\begin{align}
    \sum_{k=1}^{j-1} \left[ (\entalpha - 1)\beta (z_{\pi(k)} - z_{\pi(j)}) \right]^{\frac{1}{\entalpha - 1}} &\le \sum_{k=1}^{|S(\beta\bm{z})|} \left[ (\entalpha - 1)\beta (z_{\pi(k)} - z_{\pi(j)}) \right]^{\frac{1}{\entalpha - 1}} \nonumber\\
    &< \sum_{k=1}^{|S(\beta\bm{z})|} \left[ (\entalpha - 1)\beta z_{\pi(k)} - \tau \right]^{\frac{1}{\entalpha - 1}} = 1,
\end{align}
where the last equality comes from \eqref{eq:proof_equality}.

    \item \framebox{$\pi(j) \in S(\beta \bm{z}) \Longleftarrow \sum_{k=1}^{j-1} \left[ (\entalpha - 1)\beta (z_{\pi(j)} - z_{\pi(k)}) \right]^{\frac{1}{\entalpha - 1}} < 1$}

We show the reverse implication by showing that $\pi(j) \notin S(\beta \bm{z}) \Longrightarrow \sum_{k=1}^{j-1} \left[ (\entalpha - 1)\beta (z_{\pi(j)} - z_{\pi(k)}) \right]^{\frac{1}{\entalpha - 1}} \ge 1$. 
If $\pi(j) \notin S(\beta \bm{z})$, we must have by \eqref{eq:entmax_def} that $(\entalpha - 1) \beta z_{\pi(j)} \le \tau$. 
Therefore, we have that, for all $k \in [K]$, $(\entalpha - 1) \beta (z_{\pi(k)} - z_{\pi(j)}) \ge (\entalpha - 1) \beta z_{\pi(k)} - \tau$. 
As a consequence, 
\begin{align}
    \sum_{k=1}^{j-1} \left[ (\entalpha - 1)\beta (z_{\pi(k)} - z_{\pi(j)}) \right]^{\frac{1}{\entalpha - 1}} &\ge \sum_{k=1}^{|S(\beta\bm{z})|} \left[ (\entalpha - 1)\beta (z_{\pi(k)} - z_{\pi(j)}) \right]^{\frac{1}{\entalpha - 1}} \nonumber\\
    &\ge \sum_{k=1}^{|S(\beta\bm{z})|} \left[ (\entalpha - 1)\beta z_{\pi(k)} - \tau \right]^{\frac{1}{\entalpha - 1}} = 1.
\end{align}

\end{itemize}
\rebuttal{
\section{Additional Experimental Details}
\label{sec:extra_details}

\subsection{$\mathsf{RAPS}$}
\label{subsec:raps}
The $\mathsf{RAPS}$ procedure was implemented according to the original protocol introduced by \citet{angelopoulos2022uncertaintysetsimageclassifiers} by adapting the \href{https://github.com/aangelopoulos/conformal-prediction/blob/main/notebooks/imagenet-raps.ipynb}{code} provided by those authors. The method uses two hyperparameters to regularize the original \textit{adaptive prediction sets} procedure \citep{romano2020classificationvalidadaptivecoverage}: $\lambda_\text{reg}$ --- a regularization penalty added to discourage inclusion in the prediction set, and $k_\text{reg}$ --- the order from which classes receive the penalty term. We split the original calibration data into two sets: calibration data ($60\%$) and hyperparameter tuning data ($40\%$). For each split of the dataset, we find the pair of hyperparameters that minimizes the average prediction set size on hyperparameter tuning set, with $\lambda_\text{reg}\in \{0.001,0.01,0.1,1\}$ and $k_\text{reg} \in \{1,5,10,50\}$.

\subsection{$\mathsf{opt\text{-}entmax}$}
\label{subsec:opt_entmax}

The $\mathsf{opt\text{-}entmax}$ procedure splits the calibration data into two sets: calibration data ($60\%$) to perform conformal prediction with $\entalpha\text{-}\mathsf{entmax}$ for varying values of $\gamma\in\{1.1,1.2,1.3,1.4,1.5,1.6,1.7,1.8,1.9\}$ ; and a hyperparameter tuning set ($40\%$) where the average set size is measured and from which the optimal $\gamma$ is chosen.
}
\section{Experimental Results}
\label{sec:extra_results}

\subsection{Model Accuracy}\label{subsec:modelling}
All models were evaluated on each of the 5 calibration-test set splits. The average and standard deviation of model accuracy over splits can be found in Table \ref{tab:accuracies}, along with the calibration and set sizes for each dataset.

\begin{table}[]
    \centering
\caption{Accuracy of each trained model: calibration and test set sizes, average and standard deviation of accuracy over the 5 different splits.}
\smallskip
\def\arraystretch{1.2}
\begin{tabular}{lllllll}
\cline{2-7}
           & \multicolumn{3}{c}{\textbf{Test}} & \multicolumn{3}{c}{\textbf{Calibration}} \\ \cline{2-7} 
 &
  \multicolumn{1}{c}{\multirow{2}{*}{\textbf{Size}}} &
  \multicolumn{2}{c}{\textbf{Accuracy}} &
  \multicolumn{1}{c}{\multirow{2}{*}{\textbf{Size}}} &
  \multicolumn{2}{c}{\textbf{Accuracy}} \\ \cline{3-4} \cline{6-7} 
 &
  \multicolumn{1}{c}{} &
  \multicolumn{1}{c}{avg} &
  \multicolumn{1}{c}{std} &
  \multicolumn{1}{c}{} &
  \multicolumn{1}{c}{avg} &
  \multicolumn{1}{c}{std} \\ \hline
CIFAR10    & 6000         & 0.838        & 0.005       & 4000      & 0.841     & 0.003     \\
CIFAR100   & 6000         & 0.860        & 0.003       & 4000      & 0.858     & 0.002     \\
ImageNet   & 30000        & 0.805        & 0.001       & 20000     & 0.805     & 0.001     \\
NewsGroups & 2261         & 0.744        & 0.006       & 1508      & 0.742     & 0.004     \\ \hline
\end{tabular}
\label{tab:accuracies}
\end{table}

\subsection{Coverage}\label{subsec:extra_coverage}

The coverage results of all methods over the 5 splits of calibration is shown in Figure \ref{fig:coverage_all}.
\begin{figure}[]
    \centering
    \includegraphics[width=0.9\linewidth]{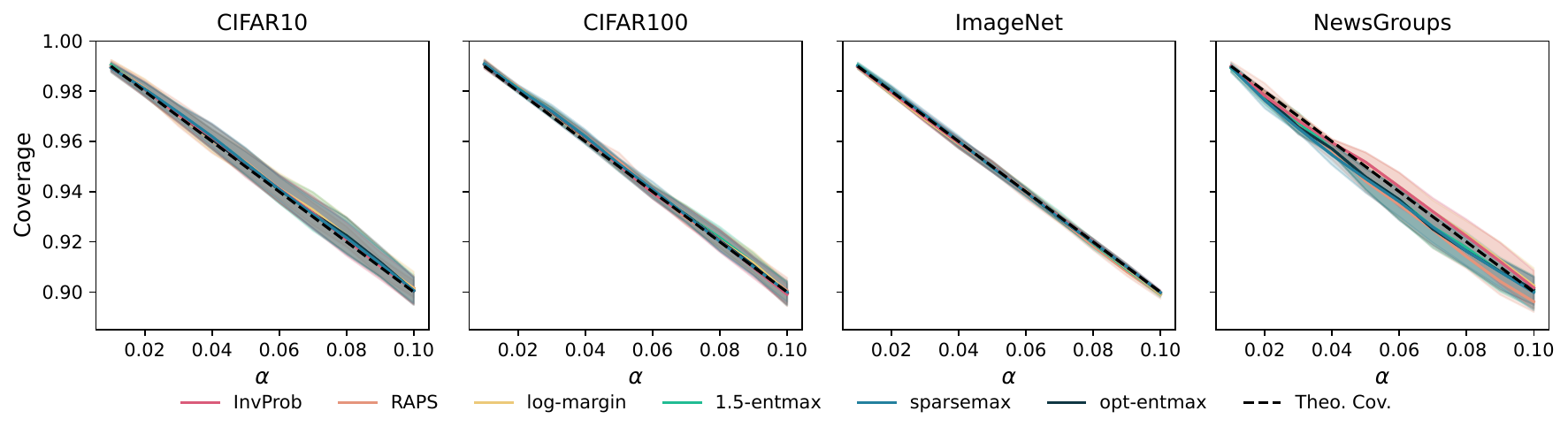}
    \caption{Coverage on the test set as a function of $\alpha$, over the 5 splits of calibration and test data.}
    \label{fig:coverage_all}
\end{figure}

\rebuttal{
\subsection{Analysis of $\mathsf{opt\text{-}entmax}$}
\label{subsec:all_optentmax_gamma}

Figure \ref{fig:fixed_alpha_lambdas} shows how varying the value of $\gamma$ for the non-conformity score affects the average prediction set size for a fixed value of $\alpha$. It is clear that the optimal value of $\gamma$ depends on both the task and the confidence level.

\begin{figure}[]
    \centering
    \includegraphics[width=0.9\linewidth]{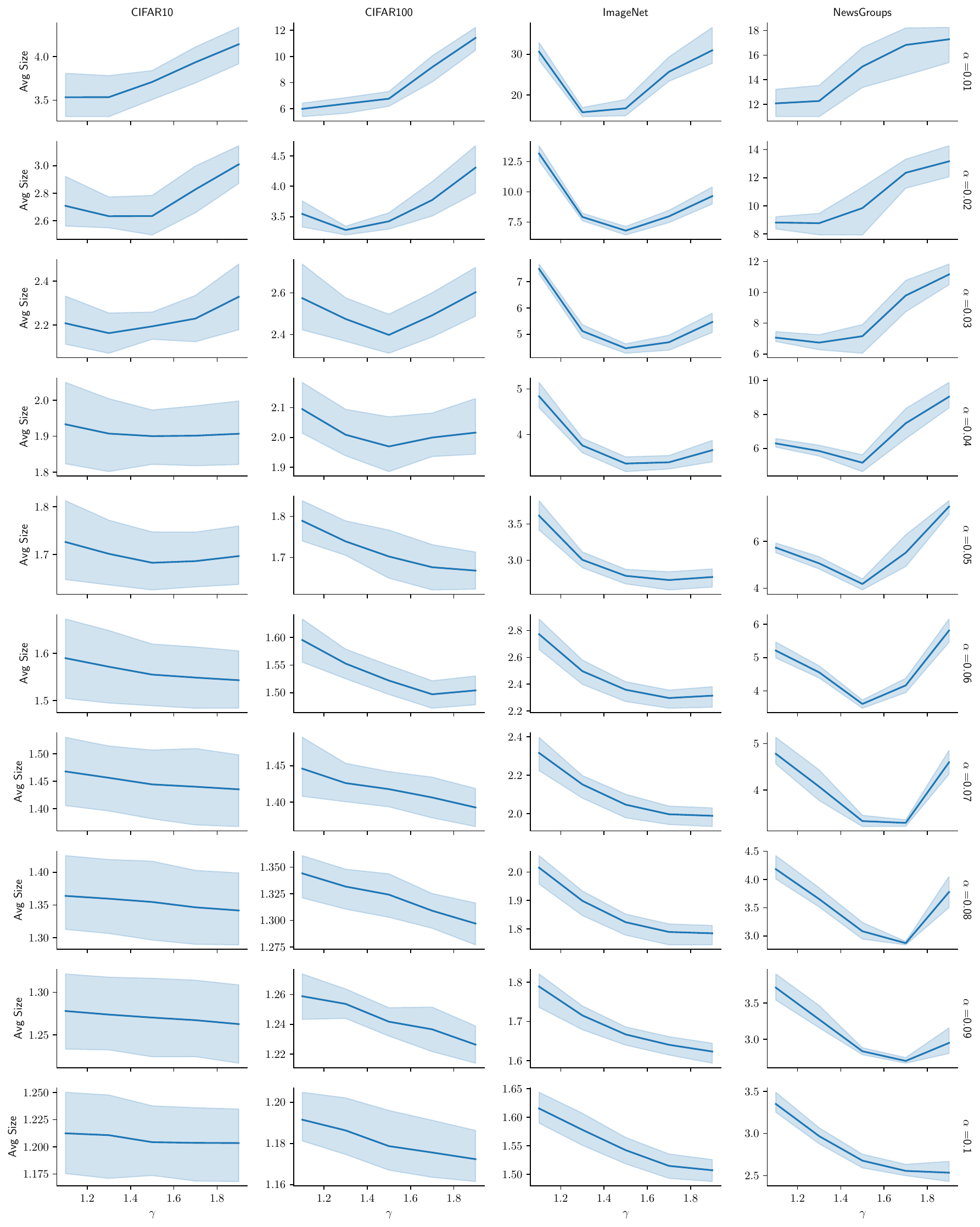}
    \caption{Average set size varying with $\gamma$ parameter, on separate calibration set for different datasets and values of $\alpha$.}
    \label{fig:fixed_alpha_lambdas}
\end{figure}
}
\subsection{Adaptiveness}\label{subsec:extra_apdaptiveness}

In Figure \ref{fig:sscv} we can see the size-stratified coverage violation (SSCV) --- introduced by \citet{angelopoulos2022uncertaintysetsimageclassifiers}, it measures the maximum deviation from the desired coverage $1-\alpha$. Partitioning the possible size cardinalities into $G$ bins, $B_1, ..., B_G$, let $\mathcal{I}_g$ be the set of observations falling in bin $g$, with $g = 1,...,G$, the SSCV of a predictor $C_\alpha$, for that bin partition is given by:

\begin{equation}
    \text{SSCV}(C, \{B_1,...,B_G\}) = \sup_g \left| \frac{\left\{i : Y_i \in \mathcal{C}_\alpha(X_i), i \in \mathcal{I}_g \right\}}{|\mathcal{J}_j|} - (1 - \alpha) \right|
\label{eq:sscv}
\end{equation}
\begin{figure}[]
    \centering
    \includegraphics[width=0.8\linewidth]{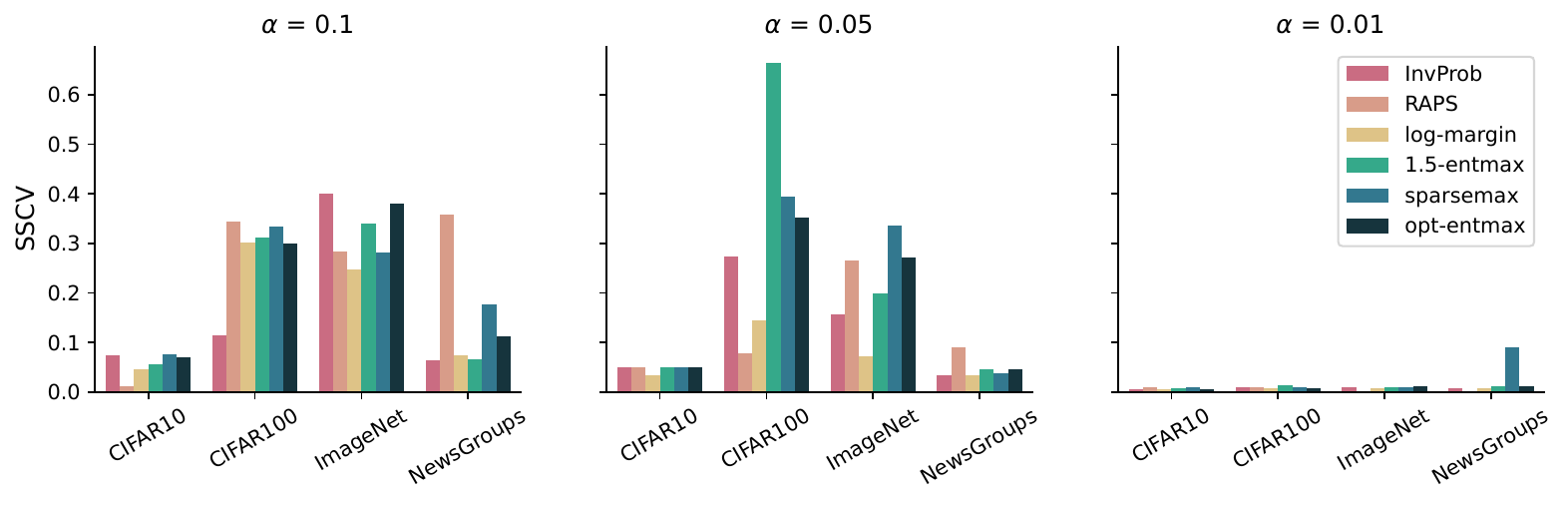}
    \caption{Size-stratified coverage violation (SSCV) for all methods and datasets, for $\alpha \in \{0.01,0.05,0.1\}$}
    \label{fig:sscv}
\end{figure}

Size-stratified coverage for the ImageNet dataset ($\alpha=0.05$) can be found in Table \ref{tab:imagenet_cov_005}. An equivalent analysis for datasets CIFAR100 and 20 NewsGroups can be found in Tables \ref{tab:coverage_cifar100} and \ref{tab:coverage_newsgroups}, respectively, for different values of $\alpha$.

\begin{table*}[]
\centering
\caption{Size-stratified coverage for the ImageNet dataset with $\alpha=0.05$.}
\smallskip 

\begin{tabular}{lrrrrrrrrrrrr}
\toprule
\multirow{2}{*}{\textbf{size}} & \multicolumn{2}{c}{\textbf{InvProb}} & \multicolumn{2}{c}{\textbf{RAPS}} & \multicolumn{2}{c}{\textbf{log-margin}} & \multicolumn{2}{c}{\textbf{1.5-entmax}} & \multicolumn{2}{c}{\textbf{sparsemax}} & \multicolumn{2}{c}{\textbf{opt-entmax}} \\
\cmidrule(lr){2-3} \cmidrule(lr){4-5} \cmidrule(lr){6-7} \cmidrule(lr){8-9} \cmidrule(lr){10-11} \cmidrule(lr){12-13}
 & n & cov & n & cov & n & cov & n & cov & n & cov & n & cov \\
        \midrule
        0 - 1  & 13745 & 0.975 & 12321 & 0.979 & 16286 & 0.966 & 13730 & 0.975 & 7260  & 0.992 & 11805 & 0.982 \\
        2 - 3  &  9579 & 0.958 &  4774 & 0.977 &  8144 & 0.945 & 10132 & 0.956 & 15693 & 0.968 & 11788 & 0.959 \\
        4 - 6  &  4189 & 0.903 &  8654 & 0.959 &  2653 & 0.912 &  3671 & 0.905 &  5681 & 0.885 &  4397 & 0.898 \\
        7 - 10 &  2046 & 0.835 &  3494 & 0.865 &  1244 & 0.916 &  1542 & 0.858 &  1226 & 0.742 &  1510 & 0.819 \\
        11 - 1000 &   440 & 0.793 &   570 & 0.684 &  1673 & 0.877 &   925 & 0.751 &   140 & 0.614 &   500 & 0.678 \\
        \bottomrule
\end{tabular}
\label{tab:imagenet_cov_005}
\end{table*}

\begin{table}
\caption{Size-stratified coverage for the CIFAR100 dataset with $\alpha=0.01$ (top), $\alpha=0.05$ (middle) and $\alpha=0.1$ (bottom).}
\centering
\begin{tabular}{lrrrrrrrrrrrr}
\toprule
\multirow{2}{*}{\textbf{size}} & \multicolumn{2}{c}{\textbf{InvProb}} & \multicolumn{2}{c}{\textbf{RAPS}} & \multicolumn{2}{c}{\textbf{log-margin}} & \multicolumn{2}{c}{\textbf{1.5-entmax}} & \multicolumn{2}{c}{\textbf{sparsemax}} & \multicolumn{2}{c}{\textbf{opt-entmax}} \\
\cmidrule(lr){2-3} \cmidrule(lr){4-5} \cmidrule(lr){6-7} \cmidrule(lr){8-9} \cmidrule(lr){10-11} \cmidrule(lr){12-13}
 & n & cov & n & cov & n & cov & n & cov & n & cov & n & cov \\
\midrule
        0 - 1  & 2222 & 0.999 & 1577 & 0.999 & 2422 & 0.998 &  795 & 0.999 &    0 & NaN   & 2397 & 0.998 \\
        2 - 3  & 1476 & 0.995 & 1015 & 0.998 & 1437 & 0.993 & 1464 & 0.999 &   12 & 1.000 & 1457 & 0.993 \\
        4 - 6 &  810 & 0.998 &  567 & 1.000 &  750 & 0.996 & 1416 & 1.000 &  190 & 1.000 &  762 & 0.997 \\
        7 - 10 &  457 & 0.991 &  898 & 0.998 &  414 & 0.990 &  977 & 0.993 & 1410 & 0.999 &  424 & 0.988 \\
        11 - 100 & 1035 & 0.980 & 1943 & 0.984 &  977 & 0.986 & 1348 & 0.976 & 4388 & 0.991 &  960 & 0.984 \\
        \bottomrule
\end{tabular}

\begin{tabular}{lrrrrrrrrrrrr}
\toprule
\multirow{2}{*}{\textbf{size}} & \multicolumn{2}{c}{\textbf{InvProb}} & \multicolumn{2}{c}{\textbf{RAPS}} & \multicolumn{2}{c}{\textbf{log-margin}} & \multicolumn{2}{c}{\textbf{1.5-entmax}} & \multicolumn{2}{c}{\textbf{sparsemax}} & \multicolumn{2}{c}{\textbf{opt-entmax}} \\
\cmidrule(lr){2-3} \cmidrule(lr){4-5} \cmidrule(lr){6-7} \cmidrule(lr){8-9} \cmidrule(lr){10-11} \cmidrule(lr){12-13}
 & n & cov & n & cov & n & cov & n & cov & n & cov & n & cov \\
        \midrule
        0 - 1 & 4089 & 0.976 & 3587 & 0.986 & 4312 & 0.970 & 4065 & 0.977 & 3634 & 0.987 & 3934 & 0.980 \\
        2 - 3  & 1357 & 0.917 & 1055 & 0.967 & 1139 & 0.913 & 1391 & 0.918 & 1954 & 0.916 & 1569 & 0.917 \\
        4 - 6  &  453 & 0.841 & 1316 & 0.872 &  346 & 0.867 &  423 & 0.830 &  393 & 0.756 &  432 & 0.806 \\
        7 - 10 &   99 & 0.677 &    0 & --   &  121 & 0.826 &  114 & 0.728 &   18 & 0.556 &   62 & 0.597 \\
        11 - 100  &    2 & 1.000 &    0 & --   &   82 & 0.805 &    7 & 0.286 &    1 & 1.000 &    3 & 0.667 \\
        \bottomrule
\end{tabular}
\begin{tabular}{lrrrrrrrrrrrr}
\toprule
\multirow{2}{*}{\textbf{size}} & \multicolumn{2}{c}{\textbf{InvProb}} & \multicolumn{2}{c}{\textbf{RAPS}} & \multicolumn{2}{c}{\textbf{log-margin}} & \multicolumn{2}{c}{\textbf{1.5-entmax}} & \multicolumn{2}{c}{\textbf{sparsemax}} & \multicolumn{2}{c}{\textbf{opt-entmax}} \\
\cmidrule(lr){2-3} \cmidrule(lr){4-5} \cmidrule(lr){6-7} \cmidrule(lr){8-9} \cmidrule(lr){10-11} \cmidrule(lr){12-13}
 & n & cov & n & cov & n & cov & n & cov & n & cov & n & cov \\
      \midrule
        0 - 1  & 5194 & 0.916 & 4456 & 0.952 & 5225 & 0.920 & 5182 & 0.924 & 5170 & 0.925 & 5170 & 0.925 \\
        2 - 3  &  786 & 0.785 & 1240 & 0.819 &  683 & 0.779 &  752 & 0.769 &  807 & 0.742 &  804 & 0.749 \\
        4 - 6 &    1 & 1.000 &  153 & 0.556 &   82 & 0.598 &   63 & 0.587 &   23 & 0.565 &   25 & 0.600 \\
        7 - 10 &    0 & --   &    0 & --   &    9 & 0.667 &    3 & 0.667 &    0 & --   &    1 & 0.000 \\
        11 - 100 &    0 & --   &    0 & --   &    1 & 1.000 &    0 & --   &    0 & --   &    0 & --   \\
        \bottomrule
\end{tabular}
\label{tab:coverage_cifar100}
\end{table}

\begin{table}
\caption{Size-stratified coverage for the 20 Newsgroups dataset with $\alpha=0.01$ (top), $\alpha=0.05$ (middle) and $\alpha=0.1$ (bottom).}
\centering
\begin{tabular}{lrrrrrrrrrrrr}
\toprule
\multirow{2}{*}{\textbf{size}} & \multicolumn{2}{c}{\textbf{InvProb}} & \multicolumn{2}{c}{\textbf{RAPS}} & \multicolumn{2}{c}{\textbf{log-margin}} & \multicolumn{2}{c}{\textbf{1.5-entmax}} & \multicolumn{2}{c}{\textbf{sparsemax}} & \multicolumn{2}{c}{\textbf{opt-entmax}} \\
\cmidrule(lr){2-3} \cmidrule(lr){4-5} \cmidrule(lr){6-7} \cmidrule(lr){8-9} \cmidrule(lr){10-11} \cmidrule(lr){12-13}
 & n & cov & n & cov & n & cov & n & cov & n & cov & n & cov \\
        \midrule
        0 - 1  &    0 & --   &    0 & --   &    0 & --   &    0 & --   &    0 & --   &    0 & --   \\
        2 - 3  &   77 & 0.987 &    0 & --   &   78 & 0.987 &    0 & --   &    0 & --   &   13 & 1.000 \\
        4 - 6  &  338 & 0.982 &    0 & --   &  345 & 0.983 &    0 & --   &    0 & --   &  279 & 0.978 \\
        7 - 10 &  718 & 0.986 &    0 & --   &  720 & 0.986 &   45 & 0.978 &   10 & 0.900 &  704 & 0.989 \\
        11 - 20  & 1128 & 0.994 & 2261 & 0.988 & 1118 & 0.994 & 2216 & 0.990 & 2251 & 0.988 & 1265 & 0.992 \\
        \bottomrule
\end{tabular}

\begin{tabular}{lrrrrrrrrrrrr}
\toprule
\multirow{2}{*}{\textbf{size}} & \multicolumn{2}{c}{\textbf{InvProb}} & \multicolumn{2}{c}{\textbf{RAPS}} & \multicolumn{2}{c}{\textbf{log-margin}} & \multicolumn{2}{c}{\textbf{1.5-entmax}} & \multicolumn{2}{c}{\textbf{sparsemax}} & \multicolumn{2}{c}{\textbf{opt-entmax}} \\
\cmidrule(lr){2-3} \cmidrule(lr){4-5} \cmidrule(lr){6-7} \cmidrule(lr){8-9} \cmidrule(lr){10-11} \cmidrule(lr){12-13}
 & n & cov & n & cov & n & cov & n & cov & n & cov & n & cov \\
        \midrule
        0 - 1  & 1045 & 0.957 &   40 & 0.900 & 1047 & 0.956 &  745 & 0.969 &    0 & NaN   &  745 & 0.969 \\
       2 - 3  &  413 & 0.935 &  451 & 0.973 &  411 & 0.937 &  660 & 0.947 &  125 & 0.944 &  660 & 0.947 \\
        4 - 6   &  166 & 0.946 & 1441 & 0.948 &  167 & 0.946 &  393 & 0.954 &  588 & 0.912 &  393 & 0.954 \\
        7 - 10  &   97 & 0.948 &  213 & 0.859 &   95 & 0.947 &  229 & 0.904 & 1126 & 0.952 &  229 & 0.904 \\
        11 - 20 &  540 & 0.983 &  116 & 0.957 &  541 & 0.983 &  234 & 0.923 &  422 & 0.981 &  234 & 0.923 \\
        \bottomrule
\end{tabular}
\begin{tabular}{lrrrrrrrrrrrr}
\toprule
\multirow{2}{*}{\textbf{size}} & \multicolumn{2}{c}{\textbf{InvProb}} & \multicolumn{2}{c}{\textbf{RAPS}} & \multicolumn{2}{c}{\textbf{log-margin}} & \multicolumn{2}{c}{\textbf{1.5-entmax}} & \multicolumn{2}{c}{\textbf{sparsemax}} & \multicolumn{2}{c}{\textbf{opt-entmax}} \\
\cmidrule(lr){2-3} \cmidrule(lr){4-5} \cmidrule(lr){6-7} \cmidrule(lr){8-9} \cmidrule(lr){10-11} \cmidrule(lr){12-13}
 & n & cov & n & cov & n & cov & n & cov & n & cov & n & cov \\

        \midrule
        0 - 1  & 1481 & 0.918 & 1095 & 0.944 & 1481 & 0.918 & 1356 & 0.933 &  575 & 0.967 & 1226 & 0.939 \\
        2 - 3   &  320 & 0.856 &  311 & 0.945 &  312 & 0.853 &  471 & 0.885 & 1336 & 0.912 &  642 & 0.894 \\
        4 - 6  &  121 & 0.868 &  794 & 0.851 &  113 & 0.867 &  208 & 0.837 &  263 & 0.722 &  234 & 0.786 \\
        7 - 10 &   92 & 0.946 &   61 & 0.541 &   77 & 0.961 &  102 & 0.833 &   27 & 0.778 &   92 & 0.793 \\
        11 - 20 &  247 & 0.964 &    0 & --   &  278 & 0.975 &  124 & 0.944 &   60 & 0.883 &   67 & 0.985 \\
        \bottomrule
\end{tabular}
\label{tab:coverage_newsgroups}
\end{table}

\end{document}